\documentclass[conference]{IEEEtran}
\IEEEoverridecommandlockouts

\usepackage{cite}
\usepackage{amsmath,amssymb,amsfonts, amsthm}
\usepackage{algorithm}
\usepackage[noend]{algpseudocode}
\usepackage{graphicx}
\usepackage{textcomp}
\usepackage{xcolor}
\usepackage{bm}
\def\BibTeX{{\rm B\kern-.05em{\sc i\kern-.025em b}\kern-.08em
    T\kern-.1667em\lower.7ex\hbox{E}\kern-.125emX}}

\newcommand{\trans}{T}
\newcommand{\matr}[1]{\mathbf{#1}}
\newcommand{\vect}[1]{\mathbf{#1}}
\newcommand{\norm}[1]{\left\lVert#1\right\rVert}
\newcommand{\Ri}{{\matr{R}_i}}
\newcommand{\rRik}[1][k]{{r^\mathcal{R}_{ i {#1} }}}
\newcommand{\SRik}{{\bm{\Lambda}^{\mathcal{R}_i}_{k}}}
\newcommand{\pars}{{\boldsymbol{\theta}}}
\newcommand{\DR}{{D_{\pars_k}^\mathcal{R}(\mathbf{x}_i)}}

\newtheorem{theorem}{Theorem}
\newtheorem{corollary}[theorem]{Corollary}

\begin{document}

\title{One-Pass Sparsified Gaussian Mixtures
\thanks{This work was supported in part by NSF GRFP award number DGE 1144083.}
}

\author{\IEEEauthorblockN{Eric Kightley}
\IEEEauthorblockA{\textit{Respond Software} \\
Mountain View, California, USA \\
kightley.1@gmail.com}
\and
\IEEEauthorblockN{Stephen Becker}
\IEEEauthorblockA{\textit{Department of Applied Mathematics} \\
\textit{University of Colorado Boulder}\\
Boulder, Colorado, USA \\
stephen.becker@colorado.edu}
}

\maketitle

\begin{abstract}
We present a one-pass sparsified Gaussian mixture model (SGMM). Given \textit{N} data points in
\textit{P} dimensions \textit{X}, the model fits \textit{K}
Gaussian distributions to \textit{X}
and (softly) classifies each point to these clusters. After paying an up-front cost of
\textit{O}(\textit{NP}\,log\,\textit{P}) to precondition the data, 
we subsample \textit{Q} entries of each data point and 
discard the full \textit{P}-dimensional data. SGMM operates in \textit{O}(\textit{KNQ})
time per iteration
for diagonal or spherical covariances, 
independent of \textit{P}, 
while estimating the model parameters in the full \textit{P}-dimensional space,
making it one-pass and hence suitable for streaming data. 
We derive the maximum likelihood estimators for the parameters in the sparsified regime,
demonstrate clustering on synthetic and real data, and show that SGMM is faster than GMM while
preserving accuracy.
\end{abstract}


\section{Introduction}

When performing clustering analysis on high-dimensional ($P$ features), 
high-volume ($N$ samples) data,
it is common to employ simple clustering schemes like $k$-means and
$k$-nearest-neighbors, particularly during data exploration and feature engineering, 
because these techniques are fast and return informative results \cite{Jain2010}.
Often each data point $\mathbf{x}_i \in \mathbb{R}^P$ will be seen only once and must then 
be discarded, necessitating \emph{one-pass} algorithms \cite{Muthukrishnan2005}. 
Further, the latent dimension $P$ may be prohibitively large or 
the rate of data acquisition may be too high to permit analysis on the full data.

We present a clustering algorithm suitable for this regime: the sparsified
Gaussian mixture model (SGMM), building on our previous work in which we developed
the sparsification scheme we use here and applied it to
$k$-means clustering \cite{Pourkamali2017}. TheGaussian mixture model, in particular when using 
diagonal or spherical covariances, is
a natural extension of $k$-means: it increases generalizability 
by taking into account cluster size and covariance and by performing soft clustering,
while still being relatively inexpensive to compute
\cite{Murphy2012}. 

SGMM works on compressed data,
such that the computation and storage costs scale with $Q \ll P$ instead of $P$, and yet the algorithm is
one-pass, meaning that the model parameters are estimated in the full $P$-dimensional
space. These requirements are seemingly
orthogonal to each other; we are able to provide both by a careful choice of how
we compress the data, which we do using a \emph{sketch}
$\matr{R}_i^\trans \mathbf{x}_i$ of size $Q \ll P$. Our sketching scheme is 
motivated by the Johnson-Lindenstrauss lemma \cite{JohnsonLindenstrauss1984}, 
which states that certain random projections into lower dimensions preserve pairwise distances
to within a small error $\varepsilon$ 
with high probability. In particular, these embeddings can be computed 
efficiently in $\mathcal{O}(N P\log P)$ time 
\cite{Ailon2009}, 
and the data are recoverable from the embeddings when they are sparse in some 
basis \cite{Candes2008a}. 

The idea is to project the data into a lower dimension and perform
analyses there, where it is cheap to do so. A variety of approaches have been proposed to this end
\cite{Achlioptas2003, Halko2011, Woodruff2014, 
Boutsidis2015c, Cohen2015}, including several applications of sketching to clustering 
algorithms \cite{Clarkson2009, Pourkamali2017, Nelson2014, Boutsidis2015c, Boutsidis2010, Cohen2015}
and Gaussian mixtures specifically \cite{Reboredo2013}. 
In general, such compressive approaches are two-pass, meaning that access to the full data
is required to estimate statistics in the original space, such as the sample mean. 
The contribution of our method is that it is compressive \emph{and} one-pass, 
meaning that we estimate statistics in the full $P$-dimensional
space using only $Q$-dimensional sketches of the data. This is possible because we use 
a different sampling matrix $\matr{R}_i$ for each data point $\mathbf{x}_i$,
so that the $Q$ features of some dense statistic $\pars$ informed by $\mathbf{x}_i$
are in general not the same as the $Q$ features informed by another data point $\mathbf{x}_j$.
Additionally, the the quantities we derive (such as mixture responsibilities and 
the Mahalanobis distance) may be useful building blocks for future algorithms using
our sketching scheme. 

The paper is organized as follows. In Section \ref{sec:theory} we discuss the theoretical foundations and
prove our main results. Then, in Section \ref{sec:algorithm} we present the sparsified Gaussian mixture model algorithm
and discuss its implementation and complexity. In Section \ref{sec:simulation} we show simulation results,
and we summarize and conclude in Section \ref{sec:conclusion}. 

\section{Theory}\label{sec:theory}

\subsection{Data Sketching}\label{subsec:sketch}

A \emph{sketch} of a matrix 
$\matr{X} = [\vect{x}_1, \vect{x}_2, \ldots, \vect{x}_N]^\trans \in \mathbb{R}^{N\times P}$ 
is a random low-dimensional projection of $\matr{X}$ \cite{Tropp2017}. Sketching is typically used to obtain a 
low-rank approximation ${\matr{X}}_Q$ with rank $Q \ll P$ such
that $\norm{{\matr{X}}_Q - \matr{X}}_F$ is minimal. There are alternative ways to find such 
an $\matr{X}_Q$, e.g. principal components analysis or feature selection, but sketching provides several
advantages \cite{Becchetti2019a}. Sketching via random projections is \emph{oblivious}, meaning
that the sketching operator can be constructed without seeing the data $\matr{X}$. This makes 
sketching particularly well-suited to streaming and distributed applications
 \cite{Tropp2019, Pourkamali2017, Jayram2013}. \emph{Cost-preserving} sketches
reduce the dimension of $\matr{X}$ in
such a way that certain cost functions (for example, Euclidean distance in the case of $k$-means clustering) 
are preserved within a low tolerance $\varepsilon$ with high probability \cite{Cohen2015}.  
More generally, sketches are often optimal in the sense that they achieve the
lower bound on some error \cite{Alon2017, Andoni2006, Jayram2013, Larsen2017, Becchetti2019a}. Finally,
sketches are fast to compute, typically $\mathcal{O}(N P \log P)$ for $N$ 
data points in $P$ dimensions \cite{Tropp2017, Ailon2009}. 

There is a broad literature on sketching,
establishing state-of-the-art bounds on the compressed dimension $Q$ in terms of  the number of data points $N$, the 
acceptable error $\varepsilon$, and the original dimension $P$, and optimizing
tradeoffs between error bounds, time, and space complexity \cite{Cohen2015, Becchetti2019a, Ailon2009}. 
There are a variety of ways to carry out the sketch in practice, including the original approach using
dense iteratively constructed Gaussian matrices, trinary matrices with sparsity $2/3$ \cite{Achlioptas2003}, and
several constructions sparser still \cite{Kane2014}. Here we use a method inspired by the 
\emph{Fast Johnson-Lindenstrauss Transform} \cite{Ailon2009} and described in detail in our previous work
\cite{Pourkamali2017}.

We will project $\mathbf{x}_i$ into a lower dimension by keeping $Q\ll P$ components chosen
uniformly at random. Before doing so we precondition the data using a random orthonormal system 
(ROS):
\begin{equation}
    \mathbf{x}_i = \matr{H}\matr{D} \mathbf{x}_i^{raw}
\end{equation}
where $\matr{D}$ is diagonal with entries $\pm 1$ chosen uniformly at random and $\matr{H}$ is 
a discrete cosine transform matrix\footnote{other choices include Hadamard or Fourier}.
The ROS transformation ensures that, with high probability, the magnitudes of the entries of 
$\mathbf{x}_i$ are relatively close to each other \cite{Do2012, Ailon2009},
minimizing the risk of ``missing'' the information in the vector when subsampling. 
The preconditioning operator $\matr{H} \matr{D}$ is the same for all $\vect{x}_i$,
and can be applied and inverted in $\mathcal{O}(NP \log P)$ time to the full dataset 
$\{ \vect{x}_1, \vect{x}_2, \ldots, \vect{x}_N\}$, which is the dominant
cost in our algorithm for small enough sketches. A detailed discussion of convergence properties
and bounds of the ROS can be found in \cite{Pourkamali2017}. Henceforth, when
we write $\mathbf{x}_i$ we assume the data have been preconditioned.

Following the preconditioning, we subsample $Q\ll P$ entries chosen uniformly at random from
$\mathbf{x}_i$.
This operation can be represented by the product $\matr{R}_i^\trans \mathbf{x}_i$
where $\matr{R}_i \in \mathbb{R}^{P\times Q}$ is sparse, with
$\matr{R}_i(p,q) = 1$ if we are keeping the $p$th feature of $\mathbf{x}_i$ and storing it
in the $q$th dimension of the sparsified vector, and 0 otherwise. Thus 
$\matr{R}_i^\trans \mathbf{x}_i \in \mathbb{R}^Q$ are the entries we preserve from $\mathbf{x}_i$.
In practice we store only the $Q$ entries of $\mathbf{x}_i$ that the subsampling picks out as well
as the indices specifying which entries were preserved 
(the indices of the $Q$ non-zero rows of $\matr{R}_i$), though it will facilitate our exposition to write
quantities like $\mathbf{R}_i \mathbf{R}_i^\trans \mathbf{x}_i \in \mathbb{R}^P$. 
Crucially, $\matr{R}_i$ is resampled for each $\mathbf{x}_i$.
This fact is what enables the method to be one-pass.

\subsection{Mixture Models}

We now describe the modeling framework, beginning with a general mixture model \cite{Murphy2012}.
Assume there are $K$ components and that each data point $\mathbf{x}_i$
belongs to one of them, indicated by the hidden variable $z_i \in \{1,2,\ldots K\}$. 
A \emph{mixture model} \cite{Murphy2012} is fully specified by
the component distributions $p_k(\mathbf{x}_i\mid\pars_k) = p(\mathbf{x}_i \mid z_i=k, \pars_k)$,
the component weights
$\bm{\pi}=\{\pi_k\}_{k=1}^K$ with $\sum \pi_k = 1$, and the 
parameters $\pars = \{\pars_k\}_{k=1}^K$. The distribution for $x_i$ is given by
\begin{equation}\label{eqn:mixture_prob}
p(\mathbf{x}_i \mid \pars_k) = \sum_{k=1}^K \pi_k p_k(\mathbf{x}_i \mid \pars_k).
\end{equation}

For a mixture of Gaussians, $\pars_k = \{\bm{\mu}_k, \matr{S}_k \}$
where $\bm{\mu}_k \in \mathbb{R}^P$ is the mean and $\matr{S}_k \in \mathbb{R}^{P\times P}$
is the covariance of the $k$th cluster, and 
$p(\matr{x}_i \mid z_i = k, \pars_k)$ is given by
\begin{equation}\label{eqn:glh_general}
    p_k(\mathbf{x}_i \mid \pars_k) = \frac{1}{(2 \pi)^{P/2}}\frac{1}{ |\matr{S}_k|^{1/2}  }
    \exp \left( -\frac{1}{2} D_{\pars_k}(\mathbf{x}_i) \right)
\end{equation}
where
\begin{equation}\label{eqn:mahalanobis}
    D_{\pars_k}(\mathbf{x}_i) = 
    \big( \mathbf{x}_i - \bm{\mu}_k \big)^\trans \bm{\Lambda}_k
    \big(\mathbf{x}_i -\bm{\mu}_k \big)
\end{equation}
is the squared Mahalanobis distance and $\bm{\Lambda}_k = \matr{S}_k^{-1}$ is the $k$th precision matrix.

The goal is to simultaneously estimate the parameters $\pars$, the weights 
$\bm{\pi}$, and the cluster assignments
$z_i$, which we do using the Expectation-Maximization algorithm.

\subsection{The EM Algorithm}

The log likelihood for data 
$\mathcal{X} = \{\mathbf{x}_1, \mathbf{x}_2, \ldots \mathbf{x}_N\}$
under the mixture distribution given in equation (\ref{eqn:mixture_prob})
is 
\begin{equation}
    \ell(\bm{\theta})   
    = \sum_{i=1}^N \log \left( \sum_{k=1}^K p_k(\vect{x}_i \mid \pars_k) \right).
\end{equation}
In the case of GMM's (as well as in many others) it is intractable to find the maximum
likelihood estimators (MLE's)
for $\bm{\theta}$ because of the hidden $\mathbf{z} = \{z_i\}_{i=1}^N$. 
Expectation-Maximization finds a local
optimum by iteratively holding one of the unknown quantities ($\pars$ or $\vect{z}$)
fixed and solving for the other. At each iteration we obtain a new estimate 
$\{\pars^t, \bm{\pi}^t\}$ computed from the previous estimate
$\{\pars^{t-1}, \bm{\pi}^{t-1}\} $.
Specifically, define the auxiliary function
\begin{equation}\label{eqn:auxiliaryfunction}
    Q(\pars, \pars^{t-1}) = E\left[ \ell_c(\pars) \mid \mathcal{X}, \pars^{t-1} \right]
\end{equation}
where
\begin{equation}
    \ell_c(\pars) = \sum_{i,k} \log p \left(\mathbf{x}_i, z_i=k \mid \pars_k \right)
\end{equation}
is the \emph{complete data log likelihood} and $\mathcal{X}$ is the dataset. 

The E step is then to compute the expected sufficient statistics in $Q$ for $\pars$, which
is equivalent to finding the \emph{responsibility} $r_{ik} = p(z_i=k\mid\mathbf{x}_i, \pars^{t-1})$
for each data point $\mathbf{x}_i$ and component $k$:
\begin{equation}\label{eqn:responsibility}
    r_{ik} = \frac{\pi_k p_k(\mathbf{x}_i \mid \pars_k^{t-1})}
                  {\sum_{j=1}^K \pi_j p_j\big(\mathbf{x}_i \mid \pars_j^{t-1}\big) }.
\end{equation}
The auxiliary function in equation (\ref{eqn:auxiliaryfunction}) can then be expressed in terms of 
the responsibility as
\begin{equation}\label{eqn:auxalt}
    Q(\pars, \pars^{t-1}) = \sum_{i,k} r_{ik}  
    \log \left[ \pi_k p_k(\mathbf{x}_i \mid \pars_k) \right].
\end{equation}

\begin{figure}
    \centering
    \includegraphics{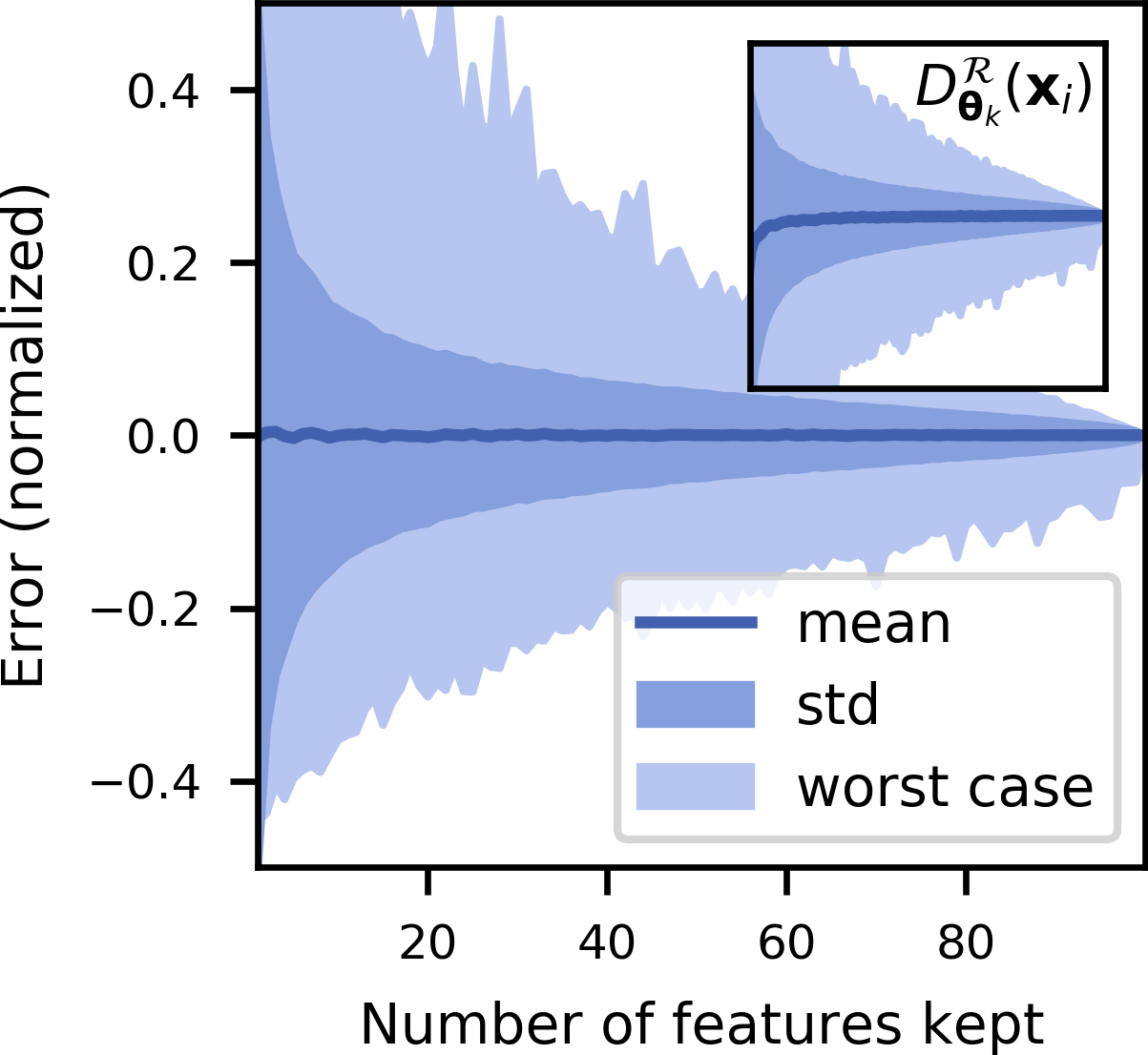}
    \caption{Error in $p_k^\mathcal{R}$ as a function of compression. 10000 
    $\mathbf{x}_i \sim \mathcal{N}(0,1)$ in 100 dimensions per trial. Inset:
    error in $\DR$.}
    \label{fig:probability_approximation}
\end{figure}

The M step is to obtain the next iterate $\{\pars^t, \bm{\pi}^t\}$ by optimizing $Q$:
\begin{equation}
    \{\bm{\theta}^t,\bm{\pi}^t\} = \text{argmax}_{\pars, \bm{\pi}} Q(\pars, \pars^{t-1}).
\end{equation}
For a mixture of Gaussians, $Q$ is optimized by the the maximum
likelihood estimators:
\begin{eqnarray}
    \widehat{\pi}_k &=& \frac{r_k}{N} \label{eqn:MLEpi}\\
    \widehat{\bm{\mu}}_k &=& \frac{\sum_i r_{ik}\mathbf{x}_i}{\sum_i r_{ik}}\label{eqn:MLEmu}  \\
    \widehat{\matr{S}}_k &=& \frac{\sum_{i} r_{ik} 
        (\mathbf{x}_i - \widehat{\bm{\mu}}_k)(\mathbf{x}_i 
            - \widehat{\bm{\mu}}_k)^\trans}{\sum_i r_{ik}}\label{eqn:MLEsigma}
\end{eqnarray}
The E and M steps are repeated until (guaranteed) convergence to a local maximum or saddle point of $Q$.

\subsection{EM for Sparsified Gaussian Mixtures}

We now present our main result, the EM algorithm for sparsified Gaussian mixtures;
i.e., the equivalents to the responsibility in equation (\ref{eqn:responsibility}) 
and the parameter MLE's in equations
(\ref{eqn:MLEpi}-\ref{eqn:MLEsigma}) under sparsification.

The sparsified analog of the squared Mahalanobis distance in equation (\ref{eqn:mahalanobis}) is
\begin{equation}\label{eqn:mahalanobis_sparsified}
    \DR = 
       \big( \mathbf{x}_i - \bm{\mu}_k \big)^\trans \SRik
        \big( \mathbf{x}_i  -\bm{\mu}_k \big)
\end{equation}
where 
\begin{equation}\label{eqn:Aik}
    \SRik = \Ri (\Ri^\trans \matr{S}_k \Ri)^{-1} \Ri^\trans \in \mathbb{R}^{P\times P}
\end{equation}
is the sparsified analog\footnote{We note that $\SRik$ is not equivalent to 
$\Ri  \Ri^\trans \bm{\Lambda}_k \Ri \Ri^\trans $;
i.e., the sparsified embedding of the precision matrix $\bm{\Lambda}_k$}
of the precision matrix $\bm{\Lambda}_k = \matr{S}_k^{-1}$.  
The sparsified Gaussian density is:
\begin{equation}\label{eqn:glh_sparsified}
    p_k^\mathcal{R}(\mathbf{x}_i \mid \pars_k) = \frac{1}{{2 \pi}^{Q/2} }
    \frac{1}{|\Ri^\trans \matr{S}_k \Ri|^{1/2}}
    \exp\left(-\frac{1}{2} D^\mathcal{R}_{\pars_k}(\mathbf{x}_i) \right).
\end{equation}
This can be taken to be a $Q$-dimensional
Gaussian with mean $\Ri\bm{\mu}_k$ and covariance 
$\Ri^\trans \matr{S}_k \Ri$ evaluated at $\Ri \mathbf{x}_i$.
Both $p^\mathcal{R}$ and $\DR$ are unbiased estimators of their dense counterparts when scaled by $P/Q$
(see Figure \ref{fig:probability_approximation}). 

The E-step is to compute the responsibility as given in equation (\ref{eqn:responsibility}).
Under sparsification, the responsbility becomes
\begin{equation}\label{eqn:responsibility_sparsified}
      \rRik =  \frac{\pi_k p_k^\mathcal{R}(\mathbf{x}_i \mid \pars_k^{t-1})}
              {\sum_{j=1}^K \pi_j p^\mathcal{R}_j\big(\mathbf{x}_i \mid \pars_k^{t-1}\big) }
\end{equation}
and hence the sparsified auxiliary function $Q$ in equation (\ref{eqn:auxalt}) is:
\begin{equation}
    Q^\mathcal{R}(\pars, \pars^{t-1}) 
    = \sum_{i,k} \rRik \log \left[ \pi_k  p_k^\mathcal{R}(\mathbf{x}_i \mid \pars_k) \right].
\end{equation}
We now derive the maximum likelihood estimators for $\pi_k$ and $\pars_k$ under sparsification.
\begin{theorem}[Maximum Likelihood Estimators for Sparsified Gaussian Mixtures]
    The maximum likelihood estimator for $\pi_k$ with respect to $Q^\mathcal{R}$ is
    \begin{equation}\label{eqn:MLEpi_sparsified}
        \widehat{{\pi}}_k^\mathcal{R} = \frac{\sum_i \rRik}{N}.
    \end{equation}
    The maximum likelihood estimators for $\bm{\mu}_k$ and $\matr{S}_k$ are the solutions to
    the system
    \begin{eqnarray}
        {\bm{\mu}}^\mathcal{R}_k &=& \left(\sum_i \rRik \SRik \right)^\dagger 
                          \sum_i \rRik \SRik \mathbf{x}_i \label{eqn:MLEmu_sparsified} \\
        \sum_i \rRik \SRik &=&
          \sum_i \rRik \SRik \matr{M}_{ik} \SRik \label{eqn:MLEsigma_sparsified}
    \end{eqnarray}
where 
\begin{equation}
    \matr{M}_{ik} = \big( \mathbf{x}_i - \bm{\mu}_k \big) \big( \mathbf{x}_i  -\bm{\mu}_k \big)^\trans.
\end{equation}
is the scatter matrix.
\end{theorem}
\begin{proof}
The component of $Q^\mathcal{R}$ with $\pi_k$-dependence is 
$$
    \ell^\mathcal{R}(\pi_k) =  \sum_{i,k} \rRik \log \pi_k
$$ from which the MLE in equation (\ref{eqn:MLEpi_sparsified}) can be derived by setting
$\partial \ell^\mathcal{R} / \partial \pi_k = 0$ for each $k$ simultanously 
and solving the resulting system.
The components of $Q^\mathcal{R}$ with $\bm{\mu}_k$ and $\matr{S}_k$ dependence are
\begin{equation}\label{eqn:ellmusigma}
    \ell^\mathcal{R}(\bm{\mu}_k, \matr{S}_k) = \sum_i \rRik \left( 
    \log |\Ri^T \matr{S}_k \Ri | + \DR 
    \right).
\end{equation}
To find $\partial \ell^\mathcal{R}/\partial \bm{\mu}_k$ we 
observe\footnote{via the chain rule and the fact that
    $\frac{\partial}{\partial \vect{a}} \big(\vect{a}^\trans \matr{A} \vect{a}) 
= \big(\matr{A}^\trans + \matr{A} \big) \vect{a}$} that
$$
    \frac{\partial}{\partial \bm{\mu}_k} \DR = -2 \SRik \big(\mathbf{x}_i - \bm{\mu}_k \big).
$$
Equation (\ref{eqn:MLEmu_sparsified}) then follows by setting 
$\partial \ell^\mathcal{R} / \partial \bm{\mu}_k = 0$ and rearranging.

We now find $\partial \ell^\mathcal{R}/\partial \matr{S}_k$. For the first term in the 
summand of equation
(\ref{eqn:ellmusigma}), we have that
\begin{equation}\label{eqn:partial_term_1}
    \frac{\partial}{\partial \matr{S}_k} \log |\Ri^\trans \matr{S}_k \Ri| = 
    \SRik
\end{equation}
which can be obtained element-wise using Jacobi's formula\footnote{Jacobi's formula states that
    $\frac{d}{dt} \det \matr{A} = \text{tr}\left[\text{adj}(\matr{A}) \frac{d\matr{A}}{dt} \right]$}
and the symmetry of $\matr{S}_k$.
For the second term, we apply the ``trace trick'':
\begin{equation}
    D_{\pars_k}^\mathcal{R} (\mathbf{x}_i) =
         \text{tr} \left[ \matr{M}_{ik} \SRik \right] \label{eqn:Dtracetrick_sparsified}
\end{equation} 
to find
\begin{equation}\label{eqn:partial_term_2}
    \frac{\partial}{\partial \matr{S}_k} D_{\pars_k}^\mathcal{R} (\mathbf{x}_i) = 
    - \SRik \matr{M}_{ik} \SRik,
\end{equation}
which can be obtained by direct element-wise differentiation of equation (\ref{eqn:Dtracetrick_sparsified}).
Setting $\partial \ell/\partial \matr{S}_k = 0$ from equation (\ref{eqn:ellmusigma}) using equations 
(\ref{eqn:partial_term_1}) and (\ref{eqn:partial_term_2}) 
we obtain equation (\ref{eqn:MLEsigma_sparsified}).
\end{proof}

\begin{figure}
    \centering
    \includegraphics{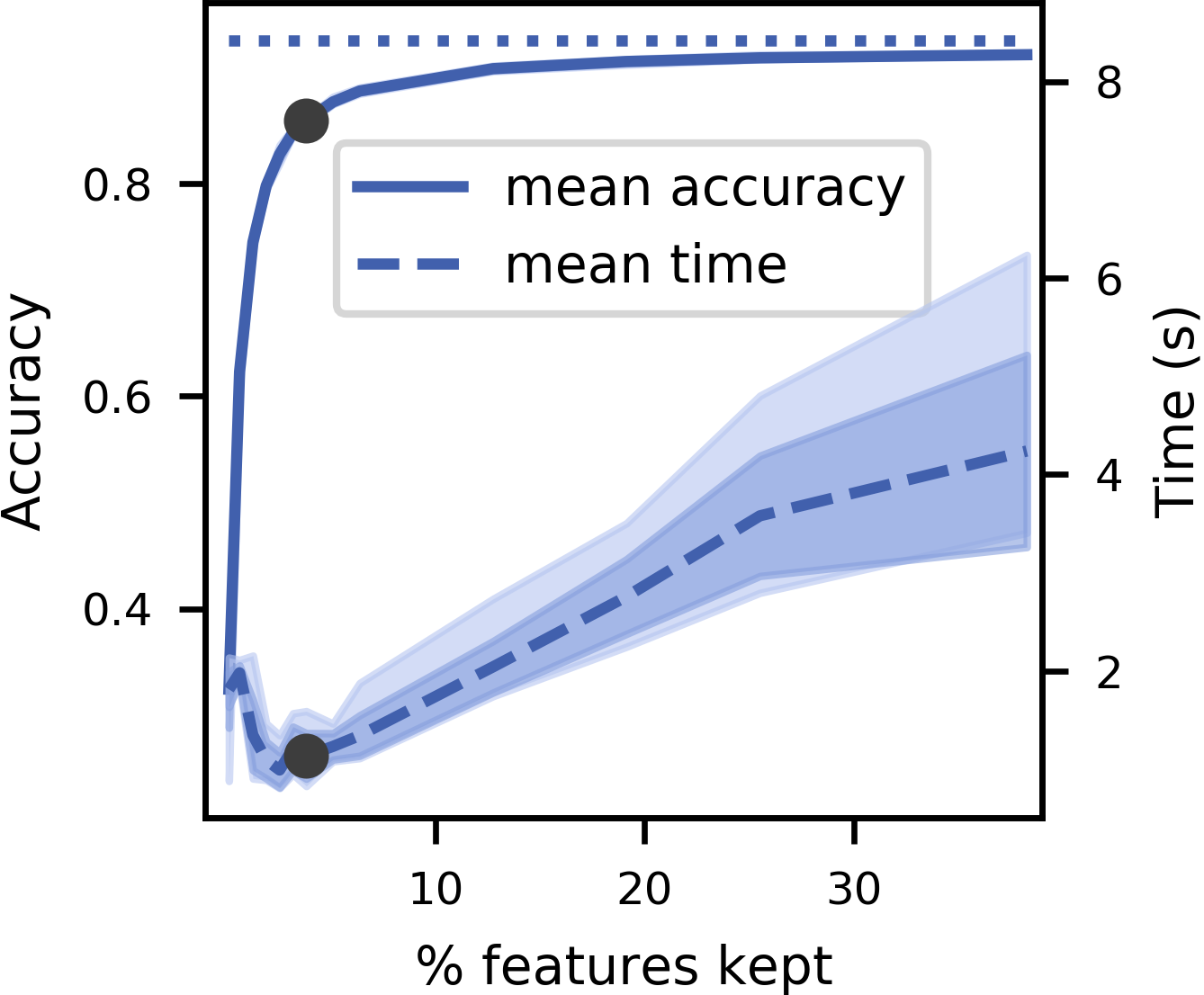} 
    \caption{Accuracy and timing of diagonal SGMM on the subset \{0,3,9\} of MNIST ($N = 18003$)
    as a function of compression.
    Three initializations per trial, 20 trials per compression. Shaded regions indicate standard deviation
    (dark) and extrema (light) taken over the trials.}
    \label{fig:mnist_accuracy}
\end{figure}

Evaluating these MLEs does not require access to the full 
$\mathbf{x}_i$, as in each case such terms are sparsified by the action of $\SRik$.
In the case of no sparsification; i.e., 
$\Ri = \matr{I}$ for all $i$, we recover the standard MLEs in equations 
(\ref{eqn:MLEpi} - \ref{eqn:MLEsigma}). Equation (\ref{eqn:MLEpi_sparsified}) has only $\pi_k^\mathcal{R}$
dependence, and hence gives the MLE for this parameter. Equation (\ref{eqn:MLEmu_sparsified}) gives 
the MLE for $\bm{\mu}_k^\mathcal{R}$ in terms of the $\SRik$. In the standard case, the $\bm{\Lambda}_k$
terms cancel and we obtain the MLE for $\bm{\mu}_k$, which is then used in place of $\bm{\mu}_k$
to find the MLE for $\matr{S}_k$; however, in the sparsified case we do not observe
this cancelation, and hence must solve equations (\ref{eqn:MLEmu_sparsified}) 
and (\ref{eqn:MLEsigma_sparsified}) simultaneously. This can be done, for example, in an EM-type 
iterative fashion, but such a procedure further requires the evaluation of $\SRik$, involving a
$Q\times Q$ inverse, of which there are $KN$ per iteration. 
These issues can be circumvented by using diagonal or spherical covariances. We give the MLEs
for the diagonal case,
$\matr{S}_k = \text{diag}(\mathbf{s}_k)$ where $\mathbf{s}_k \in \mathbb{R}^P$ (proof omitted). 
\begin{corollary}[MLEs for diagonal $\matr{S}_k$]
    When the $\matr{S}_k$ are diagonal, the
    system of equations (\ref{eqn:MLEmu_sparsified} - \ref{eqn:MLEsigma_sparsified}) yields the MLEs
\begin{eqnarray}
    \widehat{\bm{\mu}}_k &=& 
        \left(\sum_i \rRik \matr{P}_i \right)^\dagger 
        \sum_i \rRik \matr{P}_i \mathbf{x}_i \label{eqn:MLEmu_sparsified_diag}\\
    \widehat{\matr{S}}_k &=&
    \text{diag} \left[\left(\sum_i \rRik \matr{P}_i \right)^\dagger 
    \sum_i \rRik \matr{P}_i \matr{M}_{ik} \matr{P}_i \right] \label{eqn:MLEsigma_sparsified_diag}
\end{eqnarray}
    where $\matr{P}_i \in \mathbb{R}^{P\times P}$ is the sparse projection matrix:
\begin{equation}
    \matr{P}_i = \matr{R}_i \matr{R}_i^\trans.
\end{equation}
\end{corollary}

In the case of diagonal covariances
(as well as in the simpler spherical case in which $\matr{S}_k = s_k I$), 
the responsibilities $\rRik$ (E step) and the updates for 
$\widehat{\bm{\mu}}_k$ and $\widehat{\matr{S}}_k$ (M step) can each be computed in 
$\mathcal{O}(KNQ)$ time. Thus the 
EM algoritm has time complexity $\mathcal{O}(KNQ)$ per iteration, in contrast to the standard diagonal 
GMM's complexity of $\mathcal{O}(KNP)$ per iteration.

\section{Algorithm and Implementation}\label{sec:algorithm}

We now discuss the practical implementation of the SGMM 
algorithm\footnote{Code (Python/C) for SGMM and sparsified $k$-means is available at https://github.com/erickightley/sparseklearn}. 
Thus far we have used notation expedient for
mathematical exposition but not representative of how we perform computations in practice. For example,
we do not perform the matrix multiplications in quantities $\matr{R}_i \matr{H} \matr{D} \vect{x}_i$, nor do we form most of 
the intermediate quantities in our calculations. 

\begin{algorithm}
   \caption{Preconditioned Random Projection}
   \begin{algorithmic}[1]
      \State Given: Data $\{\vect{x}_1, \vect{x}_2, \ldots, \vect{x}_N\}$ with $\vect{x}_i \in \mathbb{R}^P$, 
                            sketch size $Q \ll P$
      \State Generate preconditioning operator $\matr{H} \matr{D}$;
      \For{$i = 1 \to N$}
          \State Generate projection $\matr{R}_i^\trans$;
          \State Precondition and project $\matr{y}_i \leftarrow \matr{R}_i^\trans \matr{H} \matr{D} \vect{x}_i$;
          \State Discard $\vect{x}_i$;
      \EndFor
      \State Return $\{\vect{y}_1, \ldots, \vect{y}_N\}, \{\matr{R}_i, \ldots, \matr{R}_N\}, \matr{H}, \matr{D}$;
\end{algorithmic}
\label{alg:precond}
\end{algorithm}

Algorithm \ref{alg:precond} shows the preconditioning and projection as a stand-alone algorithm. 
We need to generate and apply the preconditioning and subsampling operation 
$\matr{R}_i^\trans \matr{H} \matr{D}$
to each data point. Recall that $\matr{D} \in \mathbb{R}^{P \times P}$ is diagonal with
 entries $\pm1$ chosen uniformly at random; it therefore suffices to store just the $\mathcal{O}(P)$ indices corresponding
 to the $-1$ entries, and  computing $\matr{D} \vect{x}_i$ is a matter of flipping the sign on the relevant entries of 
 $\vect{x}_i$. Next, applying $\matr{H}$ requires no additional storage and has time complexity $\mathcal{O}(P\log P)$
 as a discrete cosine transform. Finally, the sparse projection matrix $\matr{R}_i \in \mathbb{R}^{P\times Q}$ is 
fully specified by a list of $Q$ integers indicating which of the original 
$P$ features we will keep, and the action of $\matr{R}_i$ is simply to keep the $Q$ entries of the input.

We can thus compute $\matr{R}_i^\trans \matr{H} \matr{D} \vect{x}_i$ for a single $\vect{x}_i$ 
in $\mathcal{O}(P \log P)$ time,
and need to store one list of integers of length $\mathcal{O}(P)$ for $\matr{D}$ (global across all $\vect{x}_i$)
and one list of integers of length $Q$ corresponding to the indices preserved by 
$\matr{R}_i$, in addition to the output $\vect{y}_i$, a list of floating-point numbers of length $Q$. 
Hence algorithm \ref{alg:precond} has time complexity $\mathcal{O}(N P\log P)$ and
space complexity $\mathcal{O}(NQ)$. 

Algorithm \ref{alg:sgmm} implements the Gaussian Mixture model clustering on the sparsified data. In practice we do so
for diagonal or spherical covariances.
It takes as input the sketch of the data $\{\vect{y}_1, \vect{y}_2, \ldots, \vect{y}_N\}$ with $\vect{y}_i \in \mathbb{R}^Q$
(the output of algorithm \ref{alg:precond}), the projections $\{\matr{R}_1, \matr{R}_2, \ldots, \matr{R}_N\}$,
which are stored as a total of $NQ$ integers,  the preconditioning operator $\matr{H}\matr{D}$, stored as
the $\mathcal{O}(P)$ integers representing $\matr{D}$, and the number of components $K$. It returns the
estimates of the $K$ cluster means $\widehat{\bm{\mu}}_k \in \mathbb{R}^P$, the $K$ covariances
$\widehat{\matr{S}}_k$, which are $P \times P$ matrices (if we choose to use dense covariances),
$P$-vectors (diagonal covariances), or scalars (scalar covariances), the $K$ cluster weights
 $\widehat{\pi}_k \in \mathbb{R}^K$, and the $N\times K$ responsibilities $\rRik$.

\begin{algorithm}
   \caption{Sparsified Gaussian Mixture Model (SGMM)}
   \begin{algorithmic}[1]
      \State Given: Preconditioned, subsampled 
                data $\{\vect{y}_1, \vect{y}_2, \ldots, \vect{y}_N\}$ with $\vect{y}_i \in \mathbb{R}^Q$, 
                preconditioning operator $\matr{H} \matr{D}$, projections $\{\matr{R}_1, \matr{R}_2, \ldots, \matr{R}_N\}$,
                number of components $K$
      \State Initialize $\widehat{\bm{\mu}}_k$
      \State Bootstrap $\rRik$
      \State Initialize $\widehat{\matr{S}}_k$, $\widehat{\pi}_k$
      \While{not converged}
             \State E-step: update $\rRik$ for all $i,k$
             \State M-step:  update $\widehat{\pi}_k$, $\widehat{\bm{\mu}}_k$, and  $\widehat{\matr{S}}_k$ for all $k$
      \EndWhile
      \State Invert preconditioning $\matr{H} \matr{D}$ on 
      $\widehat{\bm{\mu}}_k$ and $\widehat{\matr{S}}_k$
      \State Return $\widehat{\bm{\mu}}_k$, $\widehat{\matr{S}}_k$, $\widehat{\pi}_k$ and $\rRik$
\end{algorithmic}
\label{alg:sgmm}
\end{algorithm}

In practice we initialize the means $\widehat{\bm{\mu}}_k$ using the well-known $k$-means++ algorithm
\cite{Arthur2007}, which iteratively samples the input points with probability proportional to the squared distance
of each point to the current set of initial means. Once we have selected
a data point $\vect{y}_i$ to use as an initial mean, we project it back to $P$ dimensions; alternatively, if we have access
to the original data, we may use $\matr{H}\matr{D} \vect{x}_i$. We then bootstrap the responsibilities using hard assignment
as in $k$-means: $\rRik = \delta_{kj}$ where 
\begin{equation}
    j=\text{argmin}_{k'} \norm{\vect{y}_i - \matr{R}_{i}^\trans  \widehat{\bm{\mu}}_{k'}}_2.
\end{equation}
Using the bootstrapped $\rRik$ and the initialized $\widehat{\bm{\mu}}_k$ we initialize $\widehat{\pi}_k$ using 
(\ref{eqn:MLEpi_sparsified}) and $\widehat{\matr{S}}_k$ using (\ref{eqn:MLEsigma_sparsified_diag}).
In the E-step we update $\rRik$ with (\ref{eqn:responsibility_sparsified}), and in the M-step we
update $\widehat{\pi}_k$, $\widehat{\bm{\mu}}_k$, and  $\widehat{\matr{S}}_k$ using (\ref{eqn:MLEpi_sparsified}),
(\ref{eqn:MLEmu_sparsified_diag}), and (\ref{eqn:MLEsigma_sparsified_diag}), respectively. Finally, if desired
we invert the preconditioning by sign-flipping features according to $\matr{D}$ and then applying 
$\matr{H}^{-1}$. The E and M-steps each have time complexity $\mathcal{O}(NKQ)$, and the 
application of the inversion $\matr{H}^{-1} = \matr{H}^\trans$
has complexity $\mathcal{O}(N P \log P)$. 

\section{Simulations}\label{sec:simulation}

\subsection{Accuracy and Timing}

Figure \ref{fig:mnist_accuracy} shows the accuracy of the SGMM classifier on the subset
$\{0,3,9\}$ of the MNIST dataset as a function of the percentage of features preserved.
SGMM recovers close to full GMM accuracy with only a small number of features in a fraction
of the time. For instance, at the gray dot, SGMM with $3.82\%$ of the features preserved 
(30 out of 784) achieves a mean accuracy of $0.86$ ($92\%$ of the accuracy with all features)
in $12.9\%$ of the computation time. We further note that there is almost no variance in 
the accuracy over multiple trials, a consequence of preconditioning that we also 
observed in our sparsified $k$-means classifier \cite{Pourkamali2017}.

\medskip

\subsection{Small Cluster Recovery}

\begin{figure}
    \centering
    \includegraphics{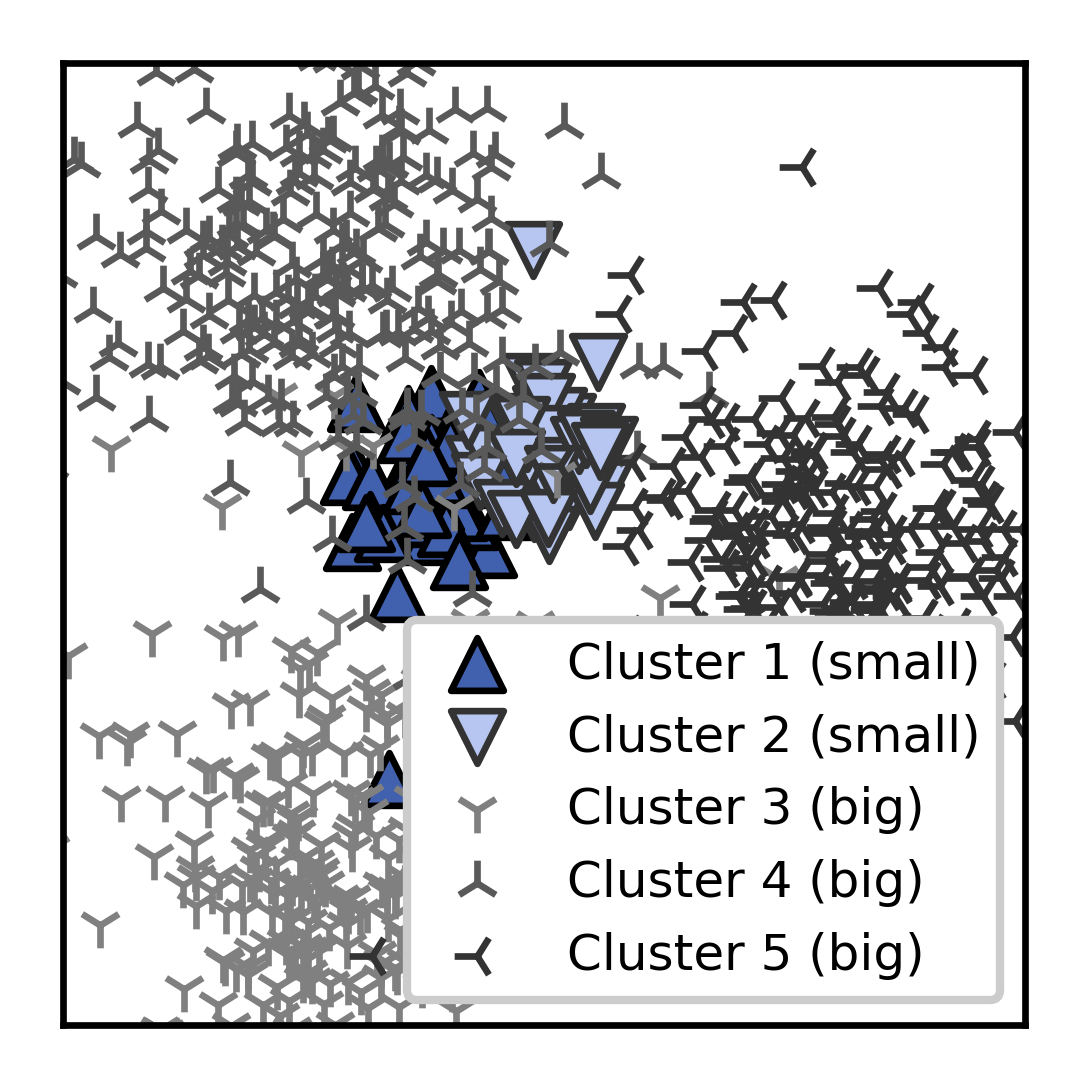}
    \caption{Small cluster recovery using spherical SGMM.}
    \label{fig:SCRG}
\end{figure}
In a regime where clusters have very different sizes, 
both in the sense of variance and number of points,
GMM (even with spherical covariance) can significantly outperform $k$-means.  
Figure \ref{fig:SCRG} shows an example in which
where SGMM correctly identifies two small clusters from three large ones 
with $98.5\%$ accuracy while $k$-means was unable to recover the clusters ($14.5\%$)
accuracy.
Data drawn from five $20$-dimensional Gaussians, embedded into $100$-dimensional space, with  
$\norm{\matr{S}_{big}}_2 \sim 10 \norm{\matr{S}_{small}}_2$ and $N_{big} = 5\times N_{small} = 250$.

\subsection{One-Pass Recovery of Means}

The fact that we resample the subsampling matrices $\matr{R}_i$ for each data point $\vect{x}_i$
is what permits our algorithm to be one-pass. Suppose we wish to approximate a quantity like 
$\norm{\vect{x}_i - \vect{x}_j} $ in the sparsified regime. The framework we use here would do so by
using only the overlapping entries preserved by both $\matr{R}_i$ and $\matr{R}_j$. In the event that 
there is no overlap, we cannot estimate such a quantity\footnote{The only place in our algorithm where
such an operation may arise is during the initialization of the SGMM algorithm using $k$-means++.}. 
We may introduce a parameter
$Q_S \in \{0,1,\ldots,Q\}$ indicating how many of the $Q$ features must overlap between all data points.
To do so in practice we sample $Q_S$ indices once during the sparsification in algorithm \ref{alg:precond}, and then 
sample $Q-Q_S$ for each $\matr{R}_i$ subsampling matrix. 

\begin{figure}
    \centering
    \includegraphics{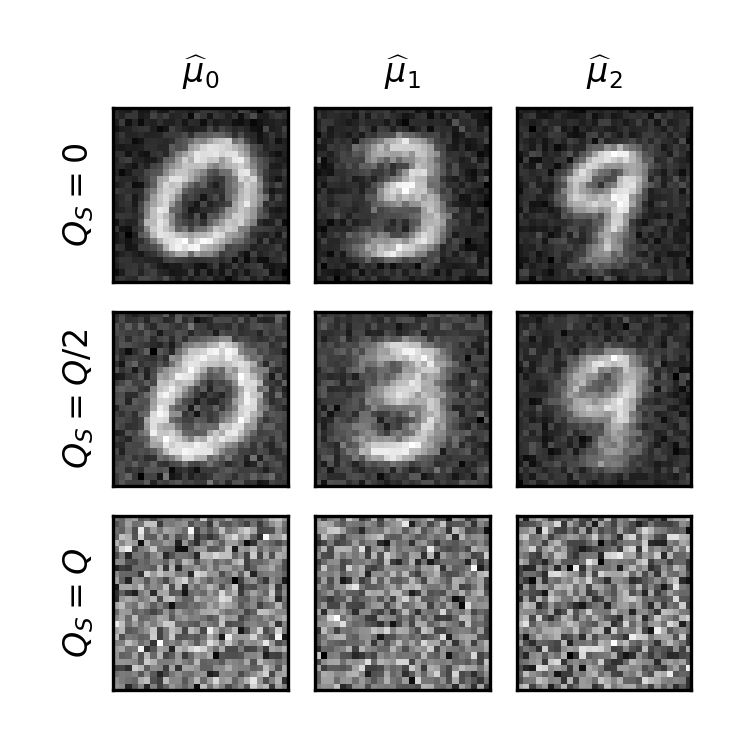} 
    \caption{Dependence of one-pass mean estimates on the number of shared features $Q_S$ in the sparsification.
                  SGMM run on the subset $\{0,3,9\}$ of MNIST with $Q=10$; spherical covariances.}
    \label{fig:onepass_Q_intro}
\end{figure}

Figure \ref{fig:onepass_Q_intro} shows the means obtained
from the SGMM algorithm applied to the MMIST subset of $\{0,3,9\}$ using three different values of $Q_S$. 
In the first row, $Q=0$, meaning that we do not enforce any overlap (though it may occur by chance), as in
algorithm (\ref{alg:sgmm}). In the second row we have set $Q_S = Q/2$, so that half of the features are shared 
between all data points. The means are noisier, because a  pixel is now half as likely to 
be preserved under sparsification given that it is not one of the $Q_S$ features saved for all $\vect{x}_i$. Finally, in
the third row we show the means when we set $Q_S=Q$; i.e., the classical random projection regime. 
The cluster assignments are accurate, but the dense means are meaningless in this regime without a second
pass through the data (see Section \ref{subsec:sketch} for further discussion). 

\section{Conclusions}\label{sec:conclusion}

The sparsified Gaussian mixture model is an efficient clustering algorithm that reduces
the storage and computational cost of Gaussian mixtures while still being one-pass.
After paying an upfront cost of 
$\mathcal{O}(N P \log P)$ to precondition the data,
SGMM compresses $N$ samples from $P$ to $Q$ dimensions, and with diagonal or spherical
covariances, fits $K$ clusters 
in $\mathcal{O}(KNQ)$ time per EM iteration.

\bibliographystyle{IEEEtran}
\bibliography{IEEEabrv,library}

\end{document}